\documentclass[11pt]{article}
\usepackage[letterpaper,margin=1in]{geometry}
\usepackage{amsmath,amssymb,amsthm,mathtools,bm,booktabs,graphicx,microtype}
\usepackage[colorlinks=true,allcolors=blue]{hyperref}
\usepackage[nameinlink,capitalize]{cleveref}
\crefname{assumption}{Assumption}{Assumptions}
\Crefname{assumption}{Assumption}{Assumptions}
\newtheorem{theorem}{Theorem}[section]
\newtheorem{lemma}[theorem]{Lemma}
\newtheorem{proposition}[theorem]{Proposition}
\newtheorem{corollary}[theorem]{Corollary}
\newtheorem{remark}[theorem]{Remark}
\newtheorem{assumption}[theorem]{Assumption}
\newtheorem{definition}[theorem]{Definition}
\newcommand{\R}{\mathbb R}
\newcommand{\E}{\mathbb E}
\newcommand{\1}{\mathbf 1}
\newcommand{\norm}[1]{\lVert #1\rVert_2}
\newcommand{\cF}{\mathcal F}
\DeclareMathOperator{\sign}{sign}

\title{Stability and Generalization of  
Straight-Through Estimators for Training Two-Layer Quantized Neural Networks}
\author{
  Yiming Ying \\
  \small School of Mathematics and Statistics \\
  \small University of Sydney \\
  \small \texttt{yiming.ying@sydney.edu.au}
}
\date{}

\begin{document}
\maketitle

\begin{abstract}
We study the identity straight-through estimator (STE) for  training a two-layer binary-activation network with hinge loss from the perspective of Statistical Learning Theory (SLT). Our central question is whether algorithmic stability can explain the statistical generalization of the estimator produced by the discontinuous STE training rule. In the saturated-output regime, the zero-initialized samplewise STE recursion is exactly  the stochastic subgradient descent on the convex latent loss $(-yu^\top x)_+$. This representation makes a stability analysis possible. We derive an exact distance identity for two coupled updates and prove approximate non-expansiveness of the common-example map, with a quadratic defect only when the two latent margins straddle zero. We then obtain explicit $\ell_2$ on-average model-stability and generalization bounds, transferring stability isometrically from the latent vector to the full first-layer matrix. Combining stability with a standard optimization bound yields an explicit excess induced-risk guarantee and the rate $O(n^{-1/2})$ when $T=n^2$. Under margin separability, a complementary argument gives the optimal-order $O(R^2/(\gamma^2n))$ expected excess misclassification error for a randomized one-pass STE iterate and a corresponding majority-vote bound. 
\end{abstract}

\section{Introduction} A major driving force behind the recent progress of artificial intelligence is
deep learning, in which deep neural networks trained on large-scale datasets
learn increasingly complex representations. Deep neural networks have achieved
remarkable performance in image classification
\cite{he2016deep,krizhevsky2017imagenet}, speech recognition
\cite{xiong2017toward}, game playing \cite{silver2016mastering}, and natural
language processing \cite{vaswani2017attention}. These successes, however, are
often accompanied by models with very large numbers of parameters and
substantial computational and memory requirements.

The high cost of inference presents a serious obstacle to deploying deep neural
networks on resource-constrained devices, including smartphones, embedded
systems, and Internet-of-Things devices. Full-precision networks may require
considerable memory for storing their parameters and a large number of
floating-point operations for each prediction. Network quantization has
therefore emerged as an important approach for reducing inference cost. It
replaces full-precision weights or activations with low-precision
representations, thereby reducing memory usage, storage, communication
bandwidth, and arithmetic cost while seeking to preserve predictive
performance.

Training a quantized neural network nevertheless presents a fundamental
mathematical difficulty. Quantized activations make the displayed network loss
piecewise constant in the trainable weights, so its ordinary gradient vanishes
almost everywhere. Straight-through estimators (STEs), introduced by Bengio,
L\'eonard, and Courville \cite{bengio2013}, address this difficulty by replacing
the unavailable activation derivative with a surrogate during
back-propagation. STEs have become the backbone  of quantization-aware
training and have achieved substantial empirical success in influential
quantized architectures such as BinaryNet
\cite{courbariaux2016binarynet}, XNOR-Net
\cite{rastegari2016xnornet}, and DoReFa-Net \cite{zhou2016dorefa}.

Despite this empirical success and widespread practical use, the theoretical
understanding of the heuristic STE training remains limited. Existing studies have focused
primarily on optimization dynamics under Gaussian data  and/or
teacher-student assumptions. Yin et al.\ \cite{yin2019} considered a
one-hidden-layer model with binarized ReLU activations and Gaussian inputs.
For suitable STE surrogates, they showed that the expected coarse gradient is
positively correlated with a population gradient and established convergence
to critical points of the population objective. Long, Yin, and Xin
\cite{long2021} considered broader classes of quantized activations and
monotone STE surrogates, proving population convergence and perfect
classification under structured-distribution and boundedness assumptions.

These population analyses provide important insight into coarse-gradient
dynamics, but they do not study an estimator trained from a finite set of training data.
Their formulation effectively assumes direct access to expectations under the
data distribution. Consequently, convergence of the population coarse-gradient
dynamics does not by itself explain whether the STE estimator trained on a finite set of training data 
generalizes well to the unseen test data. Jeong, Xin, and Yin \cite{jeong2026} recently made an important advance toward
finite-sample analysis. They established sample-complexity conditions for
convergence of STE in a two-layer network with binary weights and activations.
Their analysis assumes Gaussian covariates and a planted binary ground-truth
network generating the observations, and its objective is convergence to, or
recovery of, the planted quantized weights. Thus, their result explains how
sample size affects STE optimization under a particular generative model, but
it does not provide a generalization bound for the data-dependent STE
estimator.

Statistical generalization is a core question of Statistical Learning Theory \cite{bousquet2004introduction,mohri2018foundations,steinwart2008support,vapnik1998statistical}: how well
does the predictor produced from a finite training dataset perform on unseen
data? For STE training, this question remains largely unexplored. Moreover, it
cannot be addressed by a routine application of existing SGD theory because of the first two reasons.  Firstly, the STE
coarse gradient is generally neither the gradient nor a subgradient of the
displayed empirical network loss;  Secondly,  the discontinuous forward map makes a
direct stability analysis of the network weights difficult.

In this paper, we investigate the generalization of samplewise identity-STE
training through the framework of Algorithmic Stability \cite{bousquet2002stability,hardt2016, lei2020}. Algorithmic stability measures how
much the learned model changes when one training example is replaced and
converts this sensitivity into a generalization guarantee.  Our central idea here is that, for a zero-initialized two-layer binary-activation network, the entire identity-STE trajectory admits an exact rank-one representation. In the saturated hinge regime, this reduction enables us to construct a convex latent loss such that the associated aggregate latent vector evolves exactly according to stochastic subgradient descent. 

This latent convex loss makes an algorithmic-stability analysis
possible. We establish approximate non-expansiveness of the single-example
update map and apply the nonsmooth on-average stability approach of Lei and
Ying \cite{lei2020}. The rank-one representation then transfers stability
isometrically from the latent vector to the complete first-layer weight
matrix, without dependence on the network width. In contrast to the earlier
population and recovery analyses, our stability and induced-risk
generalization results require NEITHER Gaussian covariates NOR the existence
of a planted ground-truth neural network.

Combining stability with a stochastic subgradient optimization bound gives an
explicit vanishing excess induced-risk guarantee. Under an additional
large-margin separability condition, we also obtain the optimal-order
\[
O\!\left(\frac{R^2}{\gamma^2n}\right)
\]
expected misclassification error for a randomized one-pass STE iterate and a
corresponding majority-vote guarantee. Specifically, Our main contributions are summarized as follows. 
\begin{enumerate}
\item We identify an exact convex representation of the samplewise hinge-loss identity-STE recursion.  In the saturated regime it is unprojected stochastic subgradient descent on $F_z(u)=(-yu^\top x)_+$, and the rank-one trajectory connects the latent iterate exactly to the two-layer network parameters.  This is the structural step that makes an SLT stability analysis possible.

\item The loss convexification through the latent convex loss enables an algorithmic-stability analysis of STE.  Building on the nonsmooth on-average stability framework of Lei and Ying, we prove an exact distance identity and approximate non-expansiveness for the coupled STE updates, and derive explicit stability, generalization, and excess induced-risk bounds.  The rank-one representation transfers the stability result to the full first-layer matrix without an additional width or dimension factor, yielding an $O(n^{-1/2})$ excess induced-risk bound when $T=n^2$.

\item For the separable data with large margin, the exact reduction to perceptron iterates along  the rank-one trajectory gives an optimal-order excess misclassification result for one-pass STE.  A randomized iterate has expected excess error at most $R^2/(\gamma^2n)$, while the majority-vote classifier has error at most $2R^2/(\gamma^2n)$.  Up to universal constants, the dependence on $R$, $\gamma$, and $n$ matches the minimax lower bound \cite{hanneke2021} for realizable large-margin linear classification   when the ambient dimension is at least of order $R^2/\gamma^2$.
\end{enumerate}

\subsection{More Discussion on Related work}
\paragraph{Activation STE and coarse-gradient dynamics.}
Bengio, L\'eonard, and Courville \cite{bengio2013} proposed the straight-through heuristic for propagating information through stochastic or nondifferentiable units.  The first convergence analyses closest to our activation-quantized setting are population analyses.  Yin et al.\ \cite{yin2019} considered a one-hidden-layer convolutional model with binarized ReLU activations and Gaussian inputs.  They showed that, for a suitable ReLU STE, the expected coarse gradient is positively correlated with a population gradient and obtained convergence to critical points; they also exhibited instability of the identity STE near certain local minima.  Long, Yin, and Xin \cite{long2021} treated a broader class of quantized activations and monotone STE surrogates in a two-linear-layer classification model, proving population convergence and perfect classification under structured-distribution and boundedness assumptions.  They also made explicit the connection between identity STE and the perceptron update.  These results concern the direction and convergence of a \emph{population coarse gradient}.  Our question is different: for each finite sample, can the actual identity-STE update be represented exactly by a tractable loss, and how sensitive is the resulting randomized training algorithm to replacing one example?

\paragraph{Optimization interpretations for quantized weights.}
A complementary literature interprets STE-type weight updates through standard optimization algorithms.  Bai, Wang, and Liberty \cite{bai2019} showed that BinaryConnect can be viewed as dual averaging (or lazy projected stochastic gradient descent), gave a convex example on which it fails to converge, and proposed the proximal method ProxQuant with stationary-point guarantees.  Ajanthan et al.\ \cite{ajanthan2021} derived quantization-aware mirror maps and interpreted the continuous auxiliary weights as dual variables, with a convex convergence analysis for time-varying mirror descent.  Dockhorn et al.\ \cite{dockhorn2021} refined the dual-averaging interpretation, related BinaryConnect to generalized conditional-gradient dynamics on a smoothed dual objective, and established convergence guarantees for their ProxConnect generalization.  These papers clarify why latent full-precision weights can support quantized forward passes.  Their principal object, however, is \emph{weight} quantization and optimization convergence; they do not identify the samplewise activation-STE update studied here with subgradient descent on an induced convex loss, nor do they derive replace-one stability bounds for it.

\paragraph{Stochastic binary and finite-sample analyses.}
Shekhovtsov and Yanush \cite{shekhovtsov2021} derived principled straight-through estimators for stochastic binary networks, analyzed their estimation accuracy, and connected latent parameters to mirror descent over Bernoulli probabilities.  This is an estimator analysis for a stochastic binary model, rather than an exact convex representation of deterministic identity-STE iterates.  More recently, Jeong, Xin, and Yin \cite{jeong2026} developed sample-complexity guarantees for STE in a two-layer network with binary weights and activations, including ergodic and nonergodic convergence and behavior under label noise.  Their finite-sample objective is recovery or convergence to a global solution under a distributional model.  In contrast, our finite-sample statement is an algorithmic generalization guarantee obtained by coupling two runs on neighboring datasets; it does not assert recovery of a teacher network.

\paragraph{Algorithmic stability.}
Hardt, Recht, and Singer \cite{hardt2016} established algorithmic-stability bounds for SGD, using non-expansiveness of smooth convex gradient steps.  Lei and Ying \cite{lei2020} introduced $\ell_2$ on-average model stability and treated nonsmooth convex losses through approximate non-expansiveness.  In particular, their $\alpha=0$ result already yields an $O(n^{-1/2})$ rate when $T$ is of order $n^2$.  Existing STE theory uses the word ``stability'' mainly for optimization behavior or coarse-gradient dynamics, which is distinct from sensitivity to replacing a training example.  To the best of our knowledge, an on-average algorithmic-stability guarantee has not previously been established for the samplewise STE training recursion considered here.  Our proof is a self-contained specialization of Lei and Ying \cite{lei2020}: the STE-specific work is to prove that saturated hinge-loss identity STE for the discontinuous forward model induces exactly \eqref{eq:psi}, that its common-example map obeys \eqref{eq:mapidentity}, and that the latent stability bound transfers isometrically to the full first-layer matrix.  

\section{Two-layer model and aggregate dynamics}
Let $(x,y)\in\R^d\times\{-1,+1\}$ with $\norm{x}\le R$. Consider
\begin{equation}\label{eq:network}
f_W(x)=\frac1{\sqrt m}\sum_{j=1}^m a_j\sigma(w_j^\top x),
\qquad
\sigma(s)=\1_{\{s>0\}},
\qquad a_j\in\{-1,+1\},
\end{equation}
where the output signs are fixed. Write
\[
m_+=|\{j:a_j=+1\}|,\quad
m_-=|\{j:a_j=-1\}|,\quad
A_+=\frac{m_+}{\sqrt m},\quad
A_-=\frac{m_-}{\sqrt m},
\]
and assume $m_+,m_->0$.

Fix $\rho>0$ and consider the hinge margin loss
\[
\ell_{\mathrm h,\rho}(z)=(\rho-z)_+.
\]
We select the subgradient
\[
\ell_{\mathrm h,\rho}'(z)=
\begin{cases}
-1,&z<\rho,\\
0,&z\ge\rho,
\end{cases}
\]
where the value $0$ is chosen at the nondifferentiable point $z=\rho$.
At iteration $t$, let
\begin{equation}\label{eq:beta}
z_t=y_tf_{W_t}(x_t),
\qquad
\beta_t=-\ell_{\mathrm h,\rho}'(z_t)
=\1_{\{z_t<\rho\}}.
\end{equation}
Thus $\beta_t=1$ when the quantized forward margin is below the hinge threshold and $\beta_t=0$ otherwise. Using the identity STE $\widetilde\sigma'(s)=1$, the first-layer update is
\begin{equation}\label{eq:wupdate}
w_{j,t+1}=w_{j,t}+\eta_t\beta_ty_t\frac{a_j}{\sqrt m}x_t.
\end{equation}
Define $u_t=m^{-1/2}\sum_j a_jw_{j,t}$.

\begin{proposition}[Rank-one trajectory of identity STE]\label{prop:reduction}
Suppose $w_{j,1}=0$ for every $j$. Under the identity-STE update \eqref{eq:wupdate}, for every $j$ and $t$,
\begin{equation}\label{eq:rankone}
w_{j,t}=\frac{a_j}{\sqrt m}u_t,
\end{equation}
where the aggregate vector obeys the exact recursion
\begin{equation}\label{eq:uupdate}
u_{t+1}=u_t+\eta_t\beta_ty_tx_t.
\end{equation}
Consequently, whenever $u_t^\top x\ne0$,
\begin{equation}\label{eq:amplitudes}
f_{W_t}(x)=
\begin{cases}
A_+,&u_t^\top x>0,\\
-A_-,&u_t^\top x<0.
\end{cases}
\end{equation}
and therefore $\sign(f_{W_t}(x))=\sign(u_t^\top x)$.
\end{proposition}

\begin{proof}
First, using $a_j^2=1$ and the definition of $u_t$,
\begin{align*}
u_{t+1}
&=\frac1{\sqrt m}\sum_{j=1}^m a_jw_{j,t+1}\\
&=\frac1{\sqrt m}\sum_{j=1}^m
a_j\left(w_{j,t}+\eta_t\beta_ty_t\frac{a_j}{\sqrt m}x_t\right)\\
&=u_t+\frac{\eta_t\beta_ty_t}{m}
\sum_{j=1}^m a_j^2x_t\\
&=u_t+\eta_t\beta_ty_tx_t,
\end{align*}
which proves \eqref{eq:uupdate}.

We next prove \eqref{eq:rankone} by induction. Since $w_{j,1}=0$ for every $j$, the definition of $u_1$ gives $u_1=0$, so the identity holds at $t=1$. If it holds at time $t$, then \eqref{eq:wupdate} and \eqref{eq:uupdate} give
\begin{align*}
w_{j,t+1}
&=\frac{a_j}{\sqrt m}u_t
+\eta_t\beta_ty_t\frac{a_j}{\sqrt m}x_t\\
&=\frac{a_j}{\sqrt m}
\left(u_t+\eta_t\beta_ty_tx_t\right)\\
&=\frac{a_j}{\sqrt m}u_{t+1}.
\end{align*}
Induction proves the rank-one trajectory for every $j$ and $t$.

Finally, \eqref{eq:rankone} implies
\[
w_{j,t}^\top x=\frac{a_j}{\sqrt m}u_t^\top x.
\]
If $u_t^\top x>0$, exactly the neurons with $a_j=+1$ are active, and their total output is $m_+/\sqrt m=A_+$. If $u_t^\top x<0$, exactly the neurons with $a_j=-1$ are active, and their total output is $-m_-/\sqrt m=-A_-$. This proves \eqref{eq:amplitudes} and the classifier identity.
\end{proof}

Thus, under zero initialization, the full first-layer trajectory remains rank one: identity STE trains a single aggregate vector, and the two-layer binary classifier is exactly linear in prediction space. The result is the elementary identity-STE/perceptron reduction and is not claimed as new.

\begin{assumption}[Saturated hinge regime]\label{ass:saturated}
The fixed output layer satisfies
\begin{equation}\label{eq:saturation}
A_+\ge\rho,\qquad A_-\ge\rho.
\end{equation}
\end{assumption}

Assumption~\ref{ass:saturated} is mild and can be ensured by choosing a sufficiently wide network with balanced output signs. Indeed, if $m$ is even and
\[
m_+=m_-=\frac{m}{2},
\]
then $A_+=A_-=\sqrt m/2$. Consequently, \eqref{eq:saturation} holds whenever $m\ge 4\rho^2$. Thus, the saturated hinge regime imposes only a lower bound on the network width and does not restrict the input distribution or the trainable first-layer weights. Its role is to ensure that the hinge-loss activity indicator is determined entirely by the sign of the latent margin, which yields the convex latent representation of the identity-STE update.

\section{Exact hinge-loss convexification}
For $z=(x,y)$ define the induced latent loss
\begin{equation}\label{eq:psi}
F_z(u)=(-yu^\top x)_+.
\end{equation}

\begin{theorem}[Exact hinge convexification]\label{thm:convex}
Suppose that \cref{ass:saturated} holds and \(w_{j,1}=0\) for all
\(j\in[m]\). Let \(W_t\) be generated by identity-STE training with
outer loss \(\ell_{\mathrm h,\rho}(y_tf_{W_t}(x_t))\). Then, along the
rank-one trajectory \eqref{eq:rankone}, the associated latent vector
\(u_t\) satisfies
\begin{equation}\label{eq:latentupdate}
u_{t+1}
=u_t-\eta_t g_t,
\qquad
g_t=-y_tx_t\1_{\{y_tu_t^\top x_t\le0\}}
\in\partial F_{z_t}(u_t).
\end{equation}
Thus, the latent dynamics induced by identity STE are exactly
unprojected stochastic subgradient descent on the convex latent loss
\(F_z\) defined in \eqref{eq:psi}. At \(y_tu_t^\top x_t=0\), STE
selects the endpoint subgradient \(g_t=-y_tx_t\).
\end{theorem}

\begin{proof}
By Proposition~\ref{prop:reduction}, the identity-STE iterates remain
on the rank-one trajectory
\[
w_{j,t}=\frac{a_j}{\sqrt m}u_t,
\qquad j\in[m].
\]
Along the rank-one trajectory
\(w_{j,t}=a_j u_t/\sqrt m\), the quantized network output satisfies
\[
f_{W_t}(x_t)=
\begin{cases}
 A_+,&u_t^\top x_t>0,\\
 0,&u_t^\top x_t=0,\\
 -A_-,&u_t^\top x_t<0.
\end{cases}
\]
Since \(A_+\ge\rho\) and \(A_-\ge\rho\) by
\cref{ass:saturated}, it follows that
\[
y_tf_{W_t}(x_t)\ge\rho
\quad\Longleftrightarrow\quad
y_tu_t^\top x_t>0.
\]
Therefore, with the hinge-loss subgradient convention
\[
\beta_t=\1_{\{y_tf_{W_t}(x_t)<\rho\}},
\]
the STE update coefficient is exactly
\[
\beta_t=\1_{\{y_tu_t^\top x_t\le0\}}.
\]
In particular, the update is active precisely when the latent
classifier makes an error or a tie.

Substituting the rank-one representation into the STE update
\eqref{eq:wupdate} gives
\[
w_{j,t+1}
=\frac{a_j}{\sqrt m}
\left(
u_t+\eta_t\beta_t y_tx_t
\right).
\]
Therefore, the associated latent vector satisfies
\[
u_{t+1}
=u_t+\eta_t y_tx_t
  \1_{\{y_tu_t^\top x_t\le0\}}
=u_t-\eta_t g_t,
\]
where
\[
g_t=-y_tx_t\1_{\{y_tu_t^\top x_t\le0\}}.
\]

By the definition of the latent loss \eqref{eq:psi},
\[
\partial F_{z_t}(u)=
\begin{cases}
\{-y_tx_t\},&y_tu^\top x_t<0,\\[2mm]
\{\lambda(-y_tx_t):\lambda\in[0,1]\},
   &y_tu^\top x_t=0,\\[2mm]
\{0\},&y_tu^\top x_t>0.
\end{cases}
\]
Thus \(g_t\in\partial F_{z_t}(u_t)\) in every case. At
\(y_tu_t^\top x_t=0\), the identity STE selects the endpoint
\(g_t=-y_tx_t\). Consequently, the latent vector induced by the
rank-one STE trajectory follows exactly unprojected stochastic
subgradient descent on \(F_z\).
\end{proof}

\begin{remark}
The theorem does not assert that the coarse gradient is the gradient of the discontinuous network risk. It identifies a different convex loss whose subgradient agrees with the coarse gradient along the rank-one trajectory.
\end{remark}

\section{Algorithmic stability of the STE recursion}\label{sec:stability}
In this section, we establish algorithmic-stability bounds for STE. For any $n\in \mathbb N$, we use the conventational notation $[n]= \{1,2,\ldots, n\}$. A direct analysis of the full weight matrix
\[
W_t=(w_{j,t}:j\in[m])
\]
is difficult because the update \eqref{eq:wupdate} involves the quantized activation. Proposition~\ref{prop:reduction} resolves this difficulty by showing that identity STE follows a rank-one trajectory. In particular, for two coupled runs on neighboring datasets $S$ and $S^{(i)}$, \eqref{eq:rankone} gives
\begin{align}
\left\|W_t(S)-W_t(S^{(i)})\right\|_F^2
&=\sum_{j=1}^m\frac{a_j^2}{m}
\left\|u_t(S)-u_t(S^{(i)})\right\|^2\notag\\
&=\left\|u_t(S)-u_t(S^{(i)})\right\|^2,
\label{eq:stabilityisometry}
\end{align}
where $a_j^2=1$. Thus, the stability analysis of the full STE weight matrix reduces isometrically to that of the latent iterate $u_t$, without any dependence on the network width. Moreover, $u_t$ follows the stochastic subgradient recursion \eqref{eq:latentupdate} associated with the convex latent loss $F_z$. This loss convexification enables us to apply the algorithmic-stability approach for nonsmooth stochastic gradient methods \cite{lei2020} to STE.

To express the latent recursion as a single-example update map, define
\[
\alpha(s)=
\begin{cases}
1,&s\le0,\\
0,&s>0.
\end{cases}
\]
Then $F_z$ is convex and $R$-Lipschitz, and $\alpha$ is nonincreasing with a unit jump. Define the unprojected single-example map on $\R^d$ by
\begin{equation}\label{eq:updatemap}
\mathcal G_{\eta,z}(u)
=u+\eta\alpha(yu^\top x)yx.
\end{equation}
Indeed, the identity STE selects
\[
g_z(u)=-\alpha(yu^\top x)yx\in\partial F_z(u),
\]
and hence \eqref{eq:latentupdate} is equivalently
\[
u_{t+1}=\mathcal G_{\eta_t,z_t}(u_t).
\]
Thus, \eqref{eq:updatemap} introduces no new training rule; it is the update-map representation of the latent identity-STE recursion established in \cref{thm:convex}.

\begin{lemma}[Exact update-map identity]\label{lem:map}
Let $s=yu^\top x$, $s'=yv^\top x$, and $\Delta\alpha=\alpha(s)-\alpha(s')$. Then
\begin{equation}\label{eq:mapidentity}
\left\|u+\eta\alpha(s)yx-v-\eta\alpha(s')yx\right\|_2^2
=\norm{u-v}^2+2\eta\Delta\alpha(s-s')
+\eta^2(\Delta\alpha)^2\norm{x}^2.
\end{equation}
Consequently,
\begin{equation}\label{eq:approxnonexp}
\norm{\mathcal G_{\eta,z}(u)-\mathcal G_{\eta,z}(v)}^2
\le\norm{u-v}^2+\eta^2R^2.
\end{equation}
\end{lemma}

\begin{proof}
Expand the square and use $\langle u-v,yx\rangle=s-s'$. Since $\alpha$ is nonincreasing, $\Delta\alpha(s-s')\le0$. Dropping this term and using $|\Delta\alpha|\le1$ and $\norm{x}\le R$ gives \eqref{eq:approxnonexp}. 
\end{proof}

We now introduce the coupled stability setting. Let
\[
S=(z_1,\ldots,z_n),
\qquad
\widetilde S=(\widetilde z_1,\ldots,\widetilde z_n)
\]
be independent samples. For each $i\in[n]$, define
\[
z_j^{(i)}=
\begin{cases}
\widetilde z_i,&j=i,\\
z_j,&j\ne i,
\end{cases}
\qquad
S^{(i)}=(z_1^{(i)},\ldots,z_n^{(i)}).
\]
Thus, $S^{(i)}$ is obtained from $S$ by replacing $z_i$ with the independent observation $\widetilde z_i$.

Run STE on $S$ and $S^{(i)}$ using the same i.i.d.\ indices $I_t$, each sampled uniformly from $[n]$:
\[
u_{t+1}=\mathcal G_{\eta_t,z_{I_t}}(u_t),
\qquad
u_{t+1}^{(i)}
=\mathcal G_{\eta_t,z_{I_t}^{(i)}}(u_t^{(i)}),
\qquad u_1=u_1^{(i)}=0.
\]
Here,
\[
z_{I_t}^{(i)}=
\begin{cases}
\widetilde z_i,&I_t=i,\\
z_{I_t},&I_t\ne i.
\end{cases}
\]
Consequently, the two coupled runs use the same training example whenever $I_t\ne i$ and may use different examples only when $I_t=i$. All expectations below are taken with respect to $S$, $\widetilde S$, and the index sequence $(I_t)$.

\begin{definition}[$\ell_2$ on-average model stability]\label{def:oams}
For a randomized algorithm $A$, define
\[
\epsilon_2(A)
=\Big(\E\bigl[\frac1n\sum_{i=1}^n
\norm{A(S)-A(S^{(i)})}^2\bigr]\Big)^{1\over 2}.
\]
\end{definition}

By \eqref{eq:stabilityisometry}, it suffices to control the on-average stability of the latent iterates $\{u_t\}$ :
\[
\E\left[\frac1n\sum_{i=1}^n
\norm{u_t(S)-u_t(S^{(i)})}^2\right].
\]

\begin{theorem}[Explicit on-average stability of STE]\label{thm:stability}
Let
\[
\Delta_t
=\E\left[\frac1n\sum_{i=1}^n\norm{u_t-u_t^{(i)}}^2\right],
\qquad
C_n^2=5+\frac4n.
\]
Then, for every $t\ge1$,
\begin{equation}\label{eq:stabilityrec}
\Delta_{t+1}
\le\left(1+\frac1{n^2}\right)\Delta_t
+C_n^2R^2\eta_t^2,
\end{equation}
and therefore
\begin{equation}\label{eq:stabilityexplicit}
\Delta_{t+1}
\le C_n^2R^2
\sum_{k=1}^t\eta_k^2
\left(1+\frac1{n^2}\right)^{t-k}
\le C_n^2R^2e^{t/n^2}\sum_{k=1}^t\eta_k^2.
\end{equation}
For the weighted output
\begin{equation}\label{eq:averageiterate}
\bar u_T=\frac{\sum_{t=1}^T\eta_tu_t}{H_T},
\qquad H_T=\sum_{t=1}^T\eta_t,
\end{equation}
we have
\begin{equation}\label{eq:barstability}
\epsilon_2(\bar u_T)
\le C_n R\,e^{T/(2n^2)}
\left(\sum_{t=1}^T\eta_t^2\right)^{1/2}.
\end{equation}
Moreover, on the rank-one trajectory,
\begin{equation}\label{eq:isometry}
\left\|W_t(S)-W_t(S^{(i)})\right\|_F
=\norm{u_t(S)-u_t(S^{(i)})},
\end{equation}
so the same stability bound holds for the full first-layer matrix.
\end{theorem}

\begin{proof}
Write $\delta_t^{(i)}=\norm{u_t-u_t^{(i)}}$. Conditional on the past, if $I_t\ne i$, both runs use the same example, and \cref{lem:map} gives
\[
(\delta_{t+1}^{(i)})^2
\le(\delta_t^{(i)})^2+\eta_t^2R^2.
\]
If $I_t=i$, the two subgradients can differ, but both have norm at most $R$. The triangle inequality gives
\[
\delta_{t+1}^{(i)}
\le\delta_t^{(i)}+2\eta_tR,
\]
and hence
\[
(\delta_{t+1}^{(i)})^2
\le(\delta_t^{(i)})^2
+4\eta_tR\,\delta_t^{(i)}
+4\eta_t^2R^2.
\]
Since $\Pr(I_t=i)=1/n$, averaging over $i$, the index, and the data gives
\[
\Delta_{t+1}
\le\Delta_t+\eta_t^2R^2
+\frac{4\eta_tR}{n}\sqrt{\Delta_t}
+\frac{4\eta_t^2R^2}{n}.
\]
Here Cauchy--Schwarz was used for the averaged first moment. Young's inequality gives
\[
\frac{4\eta_tR}{n}\sqrt{\Delta_t}
\le\frac{\Delta_t}{n^2}+4\eta_t^2R^2,
\]
which proves \eqref{eq:stabilityrec}. Since $\Delta_1=0$, iteration and
$(1+n^{-2})^r\le e^{r/n^2}$ prove \eqref{eq:stabilityexplicit}.

For each $i$, convexity of the squared norm gives
\[
\norm{\bar u_T-\bar u_T^{(i)}}^2
\le\frac1{H_T}\sum_{t=1}^T
\eta_t\norm{u_t-u_t^{(i)}}^2. 
\]
Taking average over $i $ implies that  
\begin{align*}
\E\Big[{1\over n }\sum_{i=1}^n\norm{\bar u_T-\bar u_T^{(i)}}^2\Big]
& \le\frac1{H_T}\sum_{t=1}^T
\eta_t \E\Bigl[{1\over n}\sum_{i=1}^n\norm{u_t-u_t^{(i)}}^2 \Bigr]\\
& = \frac1{H_T}\sum_{t=1}^T
\eta_t \Delta_t \le \frac1{H_T}\sum_{t=1}^T
\eta_t \Big[C_n^2R^2e^{t/n^2}\sum_{k=1}^{t}\eta_k^2\Big]
\end{align*}
where the  last inequality, we have used \eqref{eq:stabilityexplicit}. Consequently, we have 
\begin{align*}
\Big[\epsilon_2(\bar u_T) \Big]^2 & = \E\Big[{1\over n } \sum_{i=1}^n
\norm{\bar u_T-\bar u_T^{(i)}}^2\Big]  \le \frac1{H_T}\sum_{t=1}^T
\eta_t \Big[C_n^2R^2e^{t/n^2}\sum_{k=1}^{t}\eta_k^2\Big]\\
& \le \frac1{H_T}\sum_{t=1}^T
\eta_t \Big[C_n^2R^2e^{T/n^2}\sum_{k=1}^{T}\eta_k^2\Big] 
= C_n^2R^2e^{T/n^2}\sum_{k=1}^{T}\eta_k^2. 
\end{align*}
This completes the proof for \eqref{eq:barstability}. 

Finally, \eqref{eq:rankone} yields
\[
\left\|W_t(S)-W_t(S^{(i)})\right\|_F^2
=\sum_{j=1}^m\frac{a_j^2}{m}
\norm{u_t(S)-u_t(S^{(i)})}^2,
\]
which is \eqref{eq:isometry}.
\end{proof}

\section{Stability-based generalization and excess-risk bounds}
Define the induced empirical and population risks
\[
\cF_S(u)=\frac1n\sum_{i=1}^nF_{z_i}(u),
\qquad
\cF(u)=\E_zF_z(u).
\]

\begin{lemma}[On-average stability implies expected generalization]\label{lem:stabilitygen}
Let $A$ be any randomized algorithm with output in $\R^d$. If every $F_z$ is $R$-Lipschitz, then
\begin{equation}\label{eq:genfromstability}
\left|\E\big[\cF(A(S))-\cF_S(A(S))\big]\right|
\le R\epsilon_2(A).
\end{equation}
\end{lemma}

\begin{proof}
For every $i$, exchangeability under swapping $z_i$ and $\widetilde z_i$ gives
\[
\E F_{\widetilde z_i}(A(S))
=\E F_{z_i}(A(S^{(i)})).
\]
Since an independent test observation has the same law as $\widetilde z_i$,
\[
\E[\cF(A(S))-\cF_S(A(S))]
=\frac1n\sum_{i=1}^n
\E\!\left[F_{z_i}(A(S^{(i)}))-F_{z_i}(A(S))\right].
\]
Lipschitz continuity, followed by Cauchy--Schwarz over $i$ and expectation, proves \eqref{eq:genfromstability}. Reversing the swap gives the absolute-value form.
\end{proof}

\begin{theorem}[Explicit induced-risk generalization bound]\label{thm:generalization}
For unprojected hinge-loss identity STE and the weighted output \eqref{eq:averageiterate},
\begin{equation}\label{eq:explicitgen}
\left|\E\big[\cF(\bar u_T)-\cF_S(\bar u_T)\big]\right|
\le C_nR^2 e^{T/(2n^2)}
\left(\sum_{t=1}^T\eta_t^2\right)^{1/2}.
\end{equation}
\end{theorem}

\begin{proof}
Apply \cref{lem:stabilitygen} and then \eqref{eq:barstability}.
\end{proof}

\begin{theorem}[Explicit excess induced-risk bound]\label{thm:excess}
For every fixed comparator $v\in\R^d$, unprojected hinge-loss identity STE satisfies
\begin{equation}\label{eq:excess}
\E\cF(\bar u_T)-\cF(v)
\le
\frac{\norm{v}^2+R^2Q_T}{2H_T}
+C_nR^2e^{T/(2n^2)}\sqrt{Q_T},
\end{equation}
where $H_T=\sum_{t=1}^T\eta_t$ and $Q_T=\sum_{t=1}^T\eta_t^2$.
\end{theorem}

\begin{proof}
For any $v\in\R^d$, the unprojected update identity and the subgradient inequality imply
\begin{align*}
\norm{u_{t+1}-v}^2
&=\norm{u_t-v-\eta_tg_t}^2\\
&= \norm{u_t-v}^2 -2\eta_t \langle u_t-v, g_t \rangle +  \eta_t^2 \|g_t\|^2\\
&\le\norm{u_t-v}^2
-2\eta_t\big(F_{z_{I_t}}(u_t)-F_{z_{I_t}}(v)\big)
+\eta_t^2R^2.
\end{align*}
where we have used the convexity of $F_{z_{I_t}}(\cdot)$, i.e., $F_{z_{I_t}}(v)-F_{z_{I_t}}(u_t) \ge \langle g_t, v-u_t \rangle.$
Conditioning on $S$ and the past replaces the stochastic loss by $\cF_S$.
Summing, using $u_1=0$, and applying convexity to \eqref{eq:averageiterate} gives
\[
\E[\cF_S(\bar u_T)-\cF_S(v)]
\le\frac{\norm{v}^2+R^2Q_T}{2H_T}.
\]
Add and subtract empirical risks. Since $v$ is fixed,
$\E\cF_S(v)=\cF(v)$, while \cref{thm:generalization} controls the output's generalization gap. Combining the two bounds proves \eqref{eq:excess}.
\end{proof}

\begin{corollary}[An explicit $n^{-1/2}$ bound]\label{cor:rate}
Fix a nonzero comparator $v\in\R^d$ and use the constant step size
\[
\eta=\frac{\norm{v}}{R T^{3/4}}.
\]
Then
\begin{equation}\label{eq:rateT}
\E\cF(\bar u_T)-\cF(v)
\le
\frac{\norm{v}R}{2T^{1/4}}
+\frac{\norm{v}R}{2T^{3/4}}
+\frac{C_n\norm{v}R}{T^{1/4}}e^{T/(2n^2)}.
\end{equation}
In particular, with $T=n^2$,
\begin{equation}\label{eq:raten}
\E\cF(\bar u_{n^2})-\cF(v)
\le
\frac{\norm{v}R}{\sqrt n}
\left(\frac12+e^{1/2}C_n\right)
+\frac{\norm{v}R}{2n^{3/2}}.
\end{equation}
\end{corollary}

\begin{proof}
For a constant step size, $H_T=T\eta$ and $Q_T=T\eta^2$. Substituting these identities and $\eta=\norm{v}/(RT^{3/4})$ into \eqref{eq:excess} gives \eqref{eq:rateT}. Setting $T=n^2$ then yields \eqref{eq:raten}.
\end{proof}

\begin{remark}
The $T=n^2$, $\eta\asymp T^{-3/4}$ balance and the resulting $O(n^{-1/2})$ rate are the $\alpha=0$ specialization of Lei and Ying \cite{lei2020}. Our contribution is not a new generic stability rate, but the proof that the STE recursion is exactly stochastic subgradient descent on $F_z$, the sharper two-slope map identity, and the isometric transfer of the resulting model-stability guarantee to $W$.
\end{remark}

\section{Excess misclassification error under separability}\label{sec:classification}
The induced-risk result does not itself imply classification calibration. To obtain a direct excess-misclassification guarantee, we now impose the assumption used only in this section and in \cref{prop:instability}.

\begin{assumption}[Large-margin separability]\label{ass:separable}
There are a unit vector $u_\star\in\R^d$ and $\gamma>0$ such that
$y u_\star^\top x\ge\gamma$ almost surely. Also $\norm{x}\le R$ almost surely.
\end{assumption}

Under \cref{ass:separable}, the exact identity-STE/perceptron reduction gives the desired classification result. The argument uses a one-pass online sequence rather than the repeated-index stability experiment of \cref{sec:stability}.

Let $z_1,\ldots,z_n$ be i.i.d.\ and run saturated hinge-loss identity STE once through the sample with a constant step size $\eta>0$:
\begin{equation}\label{eq:onepass}
u_{t+1}=u_t+\eta y_tx_t\1_{\{y_tu_t^\top x_t\le0\}},
\qquad u_1=0.
\end{equation}
Define
\[
\mathcal R_{01}(u)=\Pr\{yu^\top x\le0\},
\qquad
\mathcal R_{01}^*=\inf_{v\in\R^d}\mathcal R_{01}(v),
\]
and, for a classifier $h:\R^d\to\{-1,0,+1\}$, write
$\mathcal R_{01}(h)=\Pr\{h(x)\ne y\}$, with ties counted as errors.
Under \cref{ass:separable}, $\mathcal R_{01}(u_\star)=0$ and hence $\mathcal R_{01}^*=0$. Therefore, the excess misclassification error coincides with the misclassification error itself.

\begin{theorem}[Optimal-order one-pass STE misclassification error]\label{thm:classification}
Suppose \cref{ass:separable,ass:saturated} hold. Let $\tau$ be uniform on $\{1,\ldots,n\}$ and independent of the training sample, and define the randomized STE classifier
\[
\widehat h_{\rm G}(x)=\sign(u_\tau^\top x).
\]
Then
\begin{equation}\label{eq:gibbsclass}
\E\mathcal R_{01}(\widehat h_{\rm G})
=\E\mathcal R_{01}(u_\tau)
\le \frac{R^2}{\gamma^2n}.
\end{equation}
Moreover, the majority-vote classifier
\[
\widehat h_{\rm MV}(x)
=\sign\!\left(\sum_{t=1}^n\sign(u_t^\top x)\right)
\]
satisfies
\begin{equation}\label{eq:majorityclass}
\E\mathcal R_{01}(\widehat h_{\rm MV})
\le \frac{2R^2}{\gamma^2n}.
\end{equation}
\end{theorem}

\begin{proof}
Let
\[
M_n=\sum_{t=1}^n\1_{\{y_tu_t^\top x_t\le0\}}
\]
be the number of updates. On every update, separability gives an increase of at least $\eta\gamma$ in the component along $u_\star$, so
\[
u_{n+1}^\top u_\star\ge\eta\gamma M_n.
\]
Since an update occurs only when $y_tu_t^\top x_t\le0$,
\[
\norm{u_{t+1}}^2
=\norm{u_t}^2+2\eta y_tu_t^\top x_t+\eta^2\norm{x_t}^2
\le\norm{u_t}^2+\eta^2R^2.
\]
Summing only over update rounds yields $\norm{u_{n+1}}^2\le\eta^2R^2M_n$. Hence
\[
\eta\gamma M_n
\le u_{n+1}^\top u_\star
\le\norm{u_{n+1}}
\le\eta R\sqrt{M_n},
\]
and therefore $M_n\le R^2/\gamma^2$.

Because $u_t$ depends only on $z_1,\ldots,z_{t-1}$, it is independent of the fresh observation $z_t$. Thus
\[
\E\mathcal R_{01}(u_t)
=\E\1_{\{y_tu_t^\top x_t\le0\}}.
\]
Averaging the preceding identity over the independent uniform index $\tau$ proves \eqref{eq:gibbsclass}. Finally, if the majority vote is wrong, at least half of the constituent classifiers are wrong. Pointwise,
\[
\1_{\{\widehat h_{\rm MV}(x)\ne y\}}
\le\frac2n\sum_{t=1}^n
\1_{\{\sign(u_t^\top x)\ne y\}}.
\]
Taking expectations and using \eqref{eq:gibbsclass} proves \eqref{eq:majorityclass}.
\end{proof}

\begin{remark}
The mistake bound and online-to-batch conversion are classical in the literture of online machine learning \cite{cesa2006prediction,cesa2004generalization}. Their role here is to show that saturated hinge-loss identity STE generates exactly the perceptron iterates to which this argument applies. The theorem concerns a randomized iterate or a majority-vote ensemble; it does not automatically give the same bound for the final iterate or the averaged vector $\bar u_T$.
\end{remark}

\begin{remark}[Optimality of the margin rate]\label{rem:marginoptimal}
The order $R^2/(\gamma^2n)$ is minimax optimal, up to universal constants, over the realizable large-margin class in the nontrivial regime $n\gtrsim R^2/\gamma^2$ and $d\gtrsim R^2/\gamma^2$.  Hanneke and Kontorovich \cite{hanneke2021} established the corresponding optimal SVM margin bound, removing the logarithmic factor in earlier SVM bounds.  More generally, the effective margin dimension is
\[
k=\min\!\left\{d,\left\lfloor\frac{R^2}{\gamma^2}\right\rfloor\right\}.
\]
Indeed, $k$ orthogonal points $x_j=Re_j$ can be assigned arbitrary binary labels while remaining separable by a unit vector with margin at least $R/\sqrt{k}\ge\gamma$.  Hence the class contains a shattered subclass of size $k$, and the standard realizable-classification lower bound gives expected error of order $k/n$ for some data distribution.  Consequently, when $d\ge R^2/\gamma^2$, no learner can improve the order $R^2/(\gamma^2n)$ uniformly over all distributions satisfying \cref{ass:separable}.  Thus the STE bound in \cref{thm:classification} matches the optimal SVM margin rate in its dependence on $R$, $\gamma$, and $n$, although the algorithms and analyses are different.  The claim is an order-optimal worst-case statement; particular distributions may admit faster rates.
\end{remark}

\section*{Acknowledgment}  

The work of Yiming Ying is supported by Australian Research Council Discovery Project (DP250101359).   

\appendix

\section{Uniform zero-one stability can fail}
\begin{proposition}[Uniform $0$--$1$ instability]\label{prop:instability}
For every $n\ge1$ and every output layer satisfying \cref{ass:saturated}, there are neighboring samples $S,S'$ of size $n$, both separated by the same unit vector with margin at least $1/\sqrt2$, such that the sequential hinge-STE classifiers after $n$ updates satisfy
\[
\sup_{(x,y)}
\left|
\1_{\{\sign(u_{n+1}(S)^\top x)\ne y\}}
-\1_{\{\sign(u_{n+1}(S')^\top x)\ne y\}}
\right|=1.
\]
\end{proposition}

\begin{proof}
In $\R^2$, take the first observation as $(e_1,+1)$ in $S$ and $(e_2,+1)$ in $S'$. Let every remaining observation be $(c,+1)$, where $c=(e_1+e_2)/\sqrt2$. Both samples are separated by $c$ with margin at least $1/\sqrt2$. From zero initialization, the selected hinge subgradient gives $u_2=\eta_1e_1$ and $u'_2=\eta_1e_2$. Every common observation then has strictly positive latent margin in both runs, so all subsequent hinge-STE updates vanish. On the unit test vector $x=(e_1-e_2)/\sqrt2$, the final classifiers have opposite signs. Choosing the test label to agree with either one makes the two $0$--$1$ losses differ by one.
\end{proof}

The proposition does not contradict on-average model stability. Model stability controls an average Euclidean parameter perturbation, whereas the supremum in uniform $0$--$1$ stability may choose a test point exactly across the decision boundary.


\small
\begin{thebibliography}{99}
\bibitem{ajanthan2021}
T.~Ajanthan, K.~Gupta, P.~H.~S.~Torr, R.~Hartley, and R.~K.~Dokania.
Mirror descent view for neural network quantization.
In \emph{Proceedings of AISTATS}, volume 130, pages 2809--2817, 2021.

\bibitem{bai2019}
Y.~Bai, Y.-X.~Wang, and E.~Liberty.
ProxQuant: Quantized neural networks via proximal operators.
In \emph{Proceedings of ICLR}, 2019.


\bibitem{cesa2006prediction}
N.~Cesa-Bianchi and G.~Lugosi.
\newblock {\em Prediction, Learning, and Games}.
\newblock Cambridge University Press, 2006.


\bibitem{cesa2004generalization}
N.~Cesa-Bianchi, A.~Conconi, and C.~Gentile.
\newblock On the generalization ability of on-line learning algorithms.
\newblock {\em IEEE Transactions on Information Theory}, 50(9):2050--2057, 2004.


\bibitem{bengio2013}
Y.~Bengio, N.~L\'eonard, and A.~Courville.
Estimating or propagating gradients through stochastic neurons for conditional computation.
\emph{arXiv:1308.3432}, 2013.


\bibitem{bousquet2004introduction}
O.~Bousquet, S.~Boucheron, and G.~Lugosi.
\newblock Introduction to statistical learning theory.
\newblock In O.~Bousquet, U.~von Luxburg, and G.~R\"atsch, editors,
\emph{Advanced Lectures on Machine Learning},
Lecture Notes in Artificial Intelligence, volume 3176,
pages 169--207. Springer, 2004.

\bibitem{bousquet2002stability}
O.~Bousquet and A.~Elisseeff.
\newblock Stability and generalization.
\newblock \emph{Journal of Machine Learning Research},
  2:499--526, 2002.

\bibitem{courbariaux2016binarynet}
M. Courbariaux, I. Hubara, D. Soudry, R. El-Yaniv, and Y. Bengio.
\newblock Binarized Neural Networks: Training Deep Neural Networks with Weights and Activations Constrained to +1 or -1.
\newblock {\em arXiv preprint arXiv:1602.02830}, 2016.
\newblock \url{https://arxiv.org/abs/1602.02830}.

\bibitem{dockhorn2021}
T.~Dockhorn, Y.~Yu, E.~Sari, M.~Zolnouri, and V.~Partovi Nia.
Demystifying and generalizing BinaryConnect.
In \emph{Advances in Neural Information Processing Systems}, volume 34, 2021.

\bibitem{hardt2016}
M.~Hardt, B.~Recht, and Y.~Singer.
Train faster, generalize better: Stability of stochastic gradient descent.
In \emph{Proceedings of ICML}, volume 48, pages 1225--1234, 2016.

\bibitem{hanneke2021}
S.~Hanneke and A.~Kontorovich.
Stable sample compression schemes: New applications and an optimal SVM margin bound.
In \emph{Proceedings of the 32nd International Conference on Algorithmic Learning Theory},
volume 132 of \emph{Proceedings of Machine Learning Research}, pages 697--721, 2021.

\bibitem{he2016deep}
K.~He, X.~Zhang, S.~Ren, and J.~Sun, 
"Deep residual learning for image recognition," 
in \emph{Proc. IEEE Conf. Comput. Vis. Pattern Recognit. (CVPR)}, 
2016, pp. 770--778.
    

\bibitem{jeong2026}
H.~Jeong, J.~Xin, and P.~Yin.
Beyond discreteness: Sample complexity analysis of straight-through estimator for 1-bit quantization.
\emph{arXiv:2505.18113}, 2026.

\bibitem{lecun2015deep}
Y.~LeCun, Y.~Bengio, and G.~Hinton.
Deep learning.
\emph{Nature}, 521:436--444, 2015.


\bibitem{krizhevsky2017imagenet}
A.~Krizhevsky, I.~Sutskever, and G.~E.~Hinton,
\emph{Imagenet classification with deep convolutional neural networks},
Communications of the ACM, \textbf{60} (2017), pp.~84--90.

\bibitem{lei2020}
Y.~Lei and Y.~Ying.
Fine-grained analysis of stability and generalization for stochastic gradient descent.
In \emph{Proceedings of ICML}, volume 119, pages 5809--5819, 2020.

\bibitem{lecun2015deep}
Y.~LeCun, Y.~Bengio, and G.~Hinton.
Deep learning.
\emph{Nature}, 521:436--444, 2015.
 

\bibitem{long2021}
Z.~Long, P.~Yin, and J.~Xin.
Learning quantized neural nets by coarse gradient method for nonlinear classification.
\emph{Research in the Mathematical Sciences}, 8:48, 2021.


\bibitem{mohri2018foundations}
M.~Mohri, A.~Rostamizadeh, and A.~Talwalkar.
\newblock \emph{Foundations of Machine Learning}.
\newblock MIT Press, Cambridge, MA, second edition, 2018.

\bibitem{rastegari2016xnornet}
M. Rastegari, V. Ordonez, J. Redmon, and A. Farhadi.
\newblock Xnor-net: Imagenet classification using binary convolutional neural networks.
\newblock In {\em Proceedings of the European Conference on Computer Vision (ECCV)}, pages 525--542. Springer, Cham, 2016.
\newblock \url{https://arxiv.org/abs/1603.05279}.

\bibitem{shekhovtsov2021}
A.~Shekhovtsov and V.~Yanush.
Reintroducing straight-through estimators as principled methods for stochastic binary networks.
\emph{arXiv:2006.06880}, 2021.

\bibitem{steinwart2008support}
I.~Steinwart and A.~Christmann.
\newblock \emph{Support Vector Machines}.
\newblock Springer, New York, 2008.


\bibitem{vapnik1998statistical}
V.~N. Vapnik.
\newblock \emph{Statistical Learning Theory}.
\newblock Wiley, New York, 1998.
 

\bibitem{vaswani2017attention}
A.~Vaswani, N.~Shazeer, N.~Parmar, J.~Uszkoreit, L.~Jones, A.~N.~Gomez, \L.~Kaiser, and I.~Polosukhin,
\emph{Attention is all you need},
Advances in Neural Information Processing Systems, \textbf{30} (2017), pp.~5998--6008.

\bibitem{vershynin2020memory}
R.~Vershynin. Memory capacity of neural networks with threshold and rectified linear unit activations,
\emph{SIAM Journal on Mathematics of Data Science},
vol.~2, no.~4, pp. 1004--1033, 2020.


\bibitem{wang2018training}
N.~Wang, J.~Choi, D.~Brand, C.-Y. Chen, and K.~Gopalakrishnan,
Training deep neural networks with 8-bit floating point numbers,
in \emph{Advances in Neural Information Processing Systems} (NeurIPS),
vol.~31, 2018.

\bibitem{xiong2017toward}
W.~Xiong, J.~Droppo, X.~Huang, F.~Seide, M.~Seltzer,
A.~Stolcke, D.~Yu, and G.~Zweig.
Toward human parity in conversational speech recognition.
\emph{IEEE/ACM Transactions on Audio, Speech, and Language Processing},
25(12):2410--2423, 2017.

\bibitem{yin2019}
P.~Yin, J.~Lyu, S.~Zhang, S.~Osher, Y.~Qi, and J.~Xin.
Understanding straight-through estimator in training activation quantized neural nets.
In \emph{Proceedings of ICLR}, 2019.


\bibitem{zhou2016dorefa}
S. Zhou, Y. Wu, Z. Ni, X. Zhou, H. Wen, and Y. Zou.
\newblock Dorefa-net: Training low bitwidth convolutional neural networks with low bitwidth gradients.
\newblock {\em arXiv preprint arXiv:1606.06160}, 2016.
\newblock \url{https://arxiv.org/abs/1606.06160}.

\bibitem{silver2016mastering}
D.~Silver, A.~Huang, C.~J.~Maddison, A.~Guez, L.~Sifre,
G.~van den Driessche, J.~Schrittwieser, I.~Antonoglou,
V.~Panneershelvam, M.~Lanctot, S.~Dieleman, D.~Grewe,
J.~Nham, N.~Kalchbrenner, I.~Sutskever, T.~Lillicrap,
M.~Leach, K.~Kavukcuoglu, T.~Graepel, and D.~Hassabis.
Mastering the game of Go with deep neural networks and tree search.
\emph{Nature}, 529:484--489, 2016.

 


\end{thebibliography}
\end{document}